\documentclass{article}

    \PassOptionsToPackage{numbers, compress}{natbib}

\usepackage{microtype}
\usepackage{graphicx}
\usepackage{subfigure}
\usepackage{wrapfig}
\usepackage{caption}
\usepackage{amsmath}
\usepackage{amssymb}
\usepackage{mathtools}
\usepackage{amsthm}
\usepackage{algorithm}
\usepackage{algpseudocode}
\theoremstyle{plain}
\newtheorem{theorem}{Theorem}[section]

\newtheorem{lemma}[theorem]{Lemma}

\theoremstyle{definition}

\newtheorem{assumption}[theorem]{Assumption}
\theoremstyle{remark}

\usepackage[final]{neurips_2022}

\usepackage[utf8]{inputenc} 
\usepackage[T1]{fontenc}    
\usepackage{hyperref}       
\usepackage{url}            
\usepackage{booktabs}       
\usepackage{amsfonts}       
\usepackage{nicefrac}       
\usepackage{microtype}      
\usepackage{xcolor}         

\newcommand\Changes[1]{{#1}}

\title{Spending Thinking Time Wisely: Accelerating MCTS with Virtual Expansions}

\author{%
  Weirui Ye\thanks{{\small \texttt{ywr20@mails.tsinghua.edu.cn, gaoyangiiis@tsinghua.edu.cn}}}
  \quad
  Pieter Abbeel\thanks{{\small \texttt{pabbeel@berkeley.edu}}}
  \quad
  Yang Gao \footnotemark[1]\, \thanks{{\small \texttt{Corresponding author}}} \,$^\mathsection$ \\
  $^*$Tsinghua University, $^\dagger$UC Berkeley, $^\mathsection$ Shanghai Qi Zhi Institute
}

\begin{document}

\maketitle

\begin{abstract}
 One of the most important AI research questions is to trade off computation versus performance since ``perfect rationality" exists in theory but is impossible to achieve in practice. Recently, Monte-Carlo tree search (MCTS) has attracted considerable attention due to the significant performance improvement in various challenging domains. However, the expensive time cost during search severely restricts its scope for applications. This paper proposes the Virtual MCTS (V-MCTS), a variant of MCTS that spends more search time on harder states and less search time on simpler states adaptively. We give theoretical bounds of the proposed method and evaluate the performance and computations on $9 \times 9$ Go board games and Atari games. Experiments show that our method can achieve comparable performances to the original search algorithm while requiring less than $50\%$ search time on average. We believe that this approach is a viable alternative for tasks under limited time and resources. The code is available at \url{https://github.com/YeWR/V-MCTS.git}.
\end{abstract}

\section{Introduction}
When artificial intelligence was first studied in the 1950s, researchers have sought to find the solution to the question ``How to build an agent with perfect rationality". The term ``perfect rationality" ~\citep{carnap1962logical, newell1982knowledge, russell1994provably} here refers to the decision made with infinite amounts of computations. However, one can only solve small-scale problems without considering the practical computation time since classical search algorithms usually exhibit exponential running time. Therefore, recent AI \Changes{research} would no longer seek to achieve ``perfect rationality", but instead carefully trade-off computation versus the level of rationality. People have developed computational models like ``bounded optimality" to model these settings~\citep{russell1994provably}. The increasing level of rationality under the same computational budget has given us a lot of AI successes. Algorithms include the Monte-Carlo sampling algorithms, the variational inference algorithms, and using DNNs as universal function approximators \citep{coulom2006efficient, chaslot2008monte, gelly2011monte, silver2016mastering, hoffman2013stochastic}. 

Recently, MCTS-based RL algorithms have achieved much success, mainly on board games. The most notable achievement is that AlphaGo beats Hui Fan in 2015~\citep{silver2016mastering}. It is the first time \Changes{a computer program beat a human professional Go player}. Afterward, AlphaGo beats two top-ranking human players, Lee Sedol in 2016 and Jie Ke in 2017, the latter of which ranked first worldwide at the time. Later, MCTS-based RL algorithms \Changes{were} further extended to other board games and Atari games~\citep{schrittwieser2020mastering}. EfficientZero~\citep{ye2021mastering} significantly improves the sample efficiency of MCTS-based RL algorithms, \Changes{shedding} light on its future applications in the real world like robotics and self-driving. 

Despite the impressive performance of MCTS-based RL algorithms, they require massive \Changes{amounts of computation} to train and evaluate. For example, MuZero \citep{schrittwieser2020mastering} used 1000 TPUs trained for 12 hours to learn the game of Go, and for a single Atari game, it needs 40 TPUs to train 12 hours. Compared to previous algorithms on the Atari games benchmark, it needs around two orders of magnitude more compute. This prohibitively large computational requirement has slowed down both the further development of MCTS-based RL algorithms as well as its practical use. 

Under the hood, MCTS-based RL algorithms imagine the futures when taking different future action sequences. However, this imaging process for the current method is not computationally efficient. For example, AlphaGo needs to look ahead 1600 game states to place a single stone. On the contrary, top human professional players can only think through around 100-200 game states per minute~\citep{silver2016mastering}. Apart from the inefficiency, the current MCTS algorithm deals with easy and challenging cases with the same computational budget. However, human knows to use their time when it is most needed. 

In this paper, we aim to design new algorithms that save the computational time of the MCTS-based RL methods. 
We make three key \textbf{contributions}: (1) We present Virtual MCTS, a variant of MCTS, to approximate the vanilla MCTS search policies with less computation. Moreover, unlike previous pruning-based methods that focus on the selection or evaluation stage in MCTS, our method improves the search loop. It terminates the search iterations earlier adaptively when current states are simpler; (2) Theoretically, we provide some error bounds of the proposed method. Furthermore, the visualization results indicate that Virtual MCTS has a better computation and performance trade-off than vanilla MCTS; (3) Empirically, our method can save more than 50\% of search times on the challenging game Go $9 \times 9$ and more than 60\% on the visually complex Atari games while keeping comparable performances to those of vanilla MCTS.

\section{Related Work}

\subsection{Reinforcement Learning with MCTS}
For a long time, Computer Go has been regarded as a remarkably challenging game \citep{bouzy2001computer, cai2007computer}. Researchers attempt to use Monte-Carlo techniques that evaluate the value of the node state through random playouts \citep{bouzy2004monte, gelly2007combining, gelly2008achieving, silver2016mastering}.
\Changes{Afterward, UCT algorithms have generally been applied in Monte-Carlo tree search (MCTS) algorithms, which use UCB1 to select action at each node of the tree \citep{kocsis2006bandit}.}
Recently, MCTS-based RL methods \citep{silver2016mastering, silver2017mastering, silver2018general, schrittwieser2020mastering} have become increasingly popular and achieved super-human performances on board games because of their strong ability to search.

Modern MCTS-based RL algorithms include four stages in the \textbf{search loop}: selection, expansion, evaluation, and backpropagation. 
\Changes{The computation bottlenecks in vanilla MCTS come from the search loop, especially for the evaluation stage and the selection stage of each iteration.}
The selection stage is time-consuming when the search tree becomes wider and deeper. The evaluation stage is quite expensive because people attempt to evaluate the node value by random playouts to the end in previous researches. Due to the search loop, MCTS-based algorithms have multiple model inferences compared to other model-free RL methods like PPO \citep{schulman2017proximal} and SAC \citep{haarnoja2018soft}. 

\subsection{Acceleration of MCTS}
MCTS-based methods have proved their strong capability of solving complex games or tasks. However, the high computational cost of MCTS hinders its application to some real-time and more general scenarios. Therefore, numerous works are devoted to accelerating MCTS. For example, to make the selection stage more effective, some heuristic pruning methods \citep{gelly2006modification, wang2007modifications, sephton2014heuristic, baier2014mcts, baier2018mcts} aim to reduce the width and depth of the search tree with some heuristic functions. Furthermore, for more efficient evaluations, Lorentz \cite{lorentz2015early} proposed early playout termination of MCTS (MCTS-EPT) to stop the random playouts early and use an evaluation function to assess win or loss.
Moreover, Hsueh \textit{et al.} \cite{hsueh2016analysis} applied MCTS-EPT to the Chinese dark chess and proved its effectiveness. Afterward, similar ideas have been applied in the evaluation stage of AlphaGoZero \citep{silver2017mastering} and later MCTS-based methods \citep{silver2018general, schrittwieser2020mastering, ye2021mastering}. They evaluate the $Q$-values through a learnable evaluation network instead of running playouts to the end. 
Grill \textit{et al.} \cite{grill2020monte} propose a novel regularized policy optimization method based on AlphaZero to decrease the search budget of MCTS, which is from the optimization perspective. Danihelka \textit{et al.} \citep{danihelka2022policy} propose a policy improvement algorithm based on sampling actions without replacement, named Gumbel trick to achieve better performance when planning with few simulations.
However, these methods mentioned above focus on the specific stage of the search iteration or reduce the total budget through pruning and optimization methods, which are orthogonal to us. And few works targets at the search loop. Lan \textit{et al.} \citep{lan2020learning} propose DS-MCTS, which defines the uncertainty of MCTS and approximates it by extra DNNs with specific features for board games in training. During the evaluation, DS-MCTS will check periodically and stop the search if the state is certain.

\section{Background}

The AlphaGo series of work \citep{silver2016mastering, silver2017mastering, silver2018general, schrittwieser2020mastering} are all MCTS-based reinforcement learning algorithms. Those algorithms assume the environment transition dynamics are known or learn the environment dynamics. Based on the dynamics, they use the Monte-Carlo tree search (MCTS) as the policy improvement operator. I.e., taking in the current policy, MCTS returns a better policy with the search algorithm. The systematic search allows the MCTS-based RL algorithm to quickly improve the policy and perform much better in the setting where heavy reasoning is required. 

\subsection{MCTS}
\label{subsec:mcts}

This part briefly introduces the MCTS method implemented in reinforcement learning applications. 
As mentioned in the related works, modern MCTS-based RL algorithms include four stages in the search loop, namely selection, expansion, evaluation, and backpropagation. 

MCTS takes in the current states and generates a policy after the search loop of $N$ iterations. Here $N$ is a constant number of iterations set by the designer, regarded as the total budget. 
In the selection stage of each iteration, an action will be selected by maximizing over UCB. Specifically, AlphaZero \citep{silver2018general} and MuZero \citep{schrittwieser2020mastering} are developed based on a variant of UCB, P-UCT \citep{rosin2011multi} and have achieved great success on board games and Atari games. The formula of P-UCT is the Eq (\ref{eq:uct}): 
\begin{equation}
    \label{eq:uct}
    \begin{aligned}
        a^k &= \arg\max_{a \in \mathcal{A}} Q(s,a)+P(s,a)\frac{\sqrt{\sum_{b \in \mathcal{A}} N(s, b)}}{1+N(s,a)}  (c_1 + \log ((\sum_{b \in \mathcal{A}} N(s,b)+c_2+1) / c_2)),
    \end{aligned}
\end{equation}
where $k$ is the index of iteration, $\mathcal{A}$ is the action set, $Q(s, a)$ is the estimated Q-value, $P(s, a)$ is the policy prior obtained from neural networks, $N(s, a)$ is the visitations to select the action $a$ from the state $s$ and $c_1, c_2$ are hyper-parameters.
The output of MCTS is the visitation of each action of the root node. After $N$ search iterations, the final policy $\pi(s)$ is defined as the normalized root visitation distribution $\pi_N(s)$, where $\pi_{k}(s, a) = N(s, a) / \sum_{b \in \mathcal{A}} N(s, b) = N(s, a)/k, a \in \mathcal{A}$. For simplification, we use $\pi_k$ in place of $\pi_k(s)$ sometimes. And the detailed procedure of MCTS is introduced in Appendix.
In our method, we propose to approximate the final policy $\pi_N(s)$ with $\hat{\pi}_k(s)$, which we name as a virtual expanded policy, through a new expansion method and a termination rule. In this way, the number of iterations in MCTS can be reduced from $N$ to $k$.

\subsection{Computation Requirement}
Most of the computations in MCTS-based RL are in the MCTS procedure. Each action taken by MCTS requires $N$ times neural network evaluations, where $N$ is a constant number of iterations in the search loop. Traditional RL algorithms, such as PPO~\citep{schulman2017proximal} or DQN~\citep{mnih2015human}, only need a single neural network evaluation per action. Thus, MCTS-based RL is roughly $N$ times computationally more \Changes{expensive} than traditional RL algorithms. 
In practice, training a single Atari game needs 12 hours of computation time on 40 TPUs~\citep{schrittwieser2020mastering}. The computation need is roughly two orders of magnitude more than traditional RL algorithms~\citep{schulman2017proximal}, although the final performance of MuZero is much better. 

\section{Method}
We aim to spend more search time on harder states and less on easier states. Intuitively, human knows when to make a quick decision or a slow decision under different circumstances. Unfortunately, this situation-aware behavior is absent in current MCTS algorithms. Therefore, we propose an MCTS variant that terminates the search iteration adaptively.
It \Changes{consists of} two components: a novel expansion method named virtual expansion to estimate the final visitation based on the current partial tree; a termination rule that decides when to terminate based on the hardness of the current scenario. And we will display the adaptive mechanism through visualizations in Section \ref{sec:vis}.

\subsection{Termination Rule}
\label{sec:termination}

We propose to terminate the search loop earlier based on the current tree statistics. Intuitively, we no longer need to search further if we find that recent searches have little changes on the root visitation distribution. With this intuition in mind, we propose a simple modification to the MCTS search algorithm. As mentioned in \ref{subsec:mcts}, $\pi_k(s)$ is the policy defined by the visitations of the root state at iteration $k$. Let $\Delta_s(i, j)$ be the L1 difference of $\pi_i(s), \pi_j(s)$, namely $\Delta_s(i, j) = \left|\left|\pi_i(s) - \pi_{j}(s)\right|\right|_1$.
Then we terminate the search loop when we have searched at least $rN$ iterations and $\Delta_s(k, k/2) < \epsilon$, where $\epsilon$ is a tolerance hyper-parameter, $r \in (0,1)$ is the ratio of the minimum search budget and $N$ is the full search iterations. 
We show that under certain conditions, a bound on $\Delta_s(k, k/2)$ implies a bound on $\Delta_s(k, N)$. $\Delta_s(k, N)$ measures the distance between the current policy $\pi_k(s)$ and the oracle policy $\pi_N(s)$. In this way, $\Delta_s(k, k/2)$ reflects the hardness of the state $s$. Consequently, once the gap is small enough, it is unnecessary for more search iterations.

\subsection{Virtual Expansion in MCTS}
\label{sec:virtual_expansion}

\begin{minipage}{0.49\textwidth}
\begin{algorithm}[H]
    \centering
    \caption{Iteration of vanilla MCTS}
    \label{alg:expand}
    \begin{algorithmic}[1]
        \State Current $k$-th iteration step:
        \State {\bfseries Input:} $\mathcal{A}, P, Q_k(s, a), N_k(s, a)$
        \State Initialize: $s \leftarrow s_{\text{root}}$
        \Repeat \text{\color{red}{ do search}}
            \State $a^* \leftarrow \Changes{\text{UCB1}(Q, P, N)}$
            \State $s \leftarrow \text{next state}(s, a^*)$
        \Until{$N_k(s, a^*) = 0$}
        \State \text{{\color{red}{Evaluate}} the state value} $R(s, a)$ \text{and} $P(s, a)$
        \For{$s$ \text{along the search path}}
            \State $Q_{k+1}(s, a) = \frac{N_k(s, a) \cdot Q_k(s, a) + R(s, a)}{N_k(s, a) + 1}$ \State $N_{k+1}(s, a) = N_k(s, a) + 1$
        \EndFor
        \State \textbf{Return} $Q_{k+1}(s, a), N_{k+1}(s, a)$
    \end{algorithmic}
\end{algorithm}
\end{minipage}
\hfill
\begin{minipage}{0.49\textwidth}
\begin{algorithm}[H]
    \centering
    \caption{Iteration of MCTS with \textbf{Virtual Expansion}}\label{alg:virtual_expand}
    \begin{algorithmic}[1]
        \State Current $k$-th iteration step:
        \State {\bfseries Input:} $\mathcal{A}, P, Q_k(s, a), N_k(s, a), \hat{N}_k(s, a)$
        \\
        \If{\text{Not init } $\hat{N}_k(s, a)$} 
        \State \text{Init: } $\hat{N}_k(s, a) \leftarrow N_k(s, a)$
        \EndIf
        \\
        \State $s \leftarrow s_{\text{root}}$
        \State $a^* \leftarrow \Changes{\text{UCB1}(Q, P, \hat{N}})$
        \State $\hat{N}_k(s, a) \leftarrow \hat{N}_k(s, a) + 1$
        \\
        \State \textbf{Return} $\hat{N}_k(s, a)$
    \end{algorithmic}
\end{algorithm}
\end{minipage}

For the termination rule $\Delta_s(k, k/2) < \epsilon$, we assume $\pi_i$ and $\pi_j$ are directly comparable. 
However, they are not directly comparable because the tree is expanded with UCT. 
As the number of visits increases, the upper bound would be tighter, and the latter visits are more focused on the promising parts. Thus earlier visitation distributions (smaller iteration number) can exhibit more exploratory distribution, while latter ones (larger iteration number) are more exploitative on promising parts. 

To compare $\pi_i$ and $\pi_j$ properly, we propose a method called \textbf{virtual expansion} in place of the vanilla expansion. Briefly, it aligns two distributions by virtual UCT expansions until the constant budget $N$. When the tree is expanded at iteration $k$, it has $N-k$ iterations to go. A normal expansion would require evaluating neural network $N-k$ times for a more accurate $Q(s, a)$ estimate for each arm at the root node. Our proposed virtual expansion still expands $N-k$ times according to UCT, but it ignores the $N-k$ neural network evaluations and assumes that each arm's $Q(s, a)$ does not change. We denote the virtual expanded distribution from $\pi_i$ as a \textbf{virtual expanded policy} $\hat{\pi}_i$.
By doing virtual expansions on both $\pi_i$ and $\pi_j$, we will obtain the corresponding virtual expanded policies $\hat{\pi}_i, \hat{\pi}_j$. 
Here we effectively remove the different levels of exploration/exploitation in the two policies.
Then the termination condition becomes the difference of virtual expanded policies. We name the rule as \textbf{VET-Rule} (Virtual Expanded Termination Rule):
\begin{equation}
\hat{\Delta}_s(k, k/2) = \left|\left| \hat{\pi}_k(s) - \hat{\pi}_{k/2}(s) \right|\right| < \epsilon.
\end{equation}

The comparisons between vanilla expansion and virtual expansion are illustrated in \Changes{Algorithm \ref{alg:expand} and \ref{alg:virtual_expand}}. The time-consuming computations are highlighted in Algorithm \ref{alg:expand}. Line 4 to 7 in Algorithm \ref{alg:expand} target at searching with UCT to reach an unvisited state for exploration. Then it evaluates the state and backpropagates along the search path to better estimate $Q$-values. After total $N$ iterations, the visitation distribution of the root node $\pi_N(s)$ is considered as the final policy $\pi(s)$. However, in virtual expansion, listed in Algorithm \ref{alg:virtual_expand}, it only searches one step from the root node. And it selects actions based on the current estimations without changing any properties of the search tree. Furthermore, the virtual visited counts $\hat{N}_k(s, a)$ are changed after virtual visits to balance the exploitation and the exploration issue. After $N - k$ times virtual expansion, the virtual expanded policy becomes $\hat{\pi}_k(s, a) = \hat{N}_k(s, a) / N$ instead of $N_k(s, a) / k$. When $k=N$, further searches after the root have no effects on the final policy. So $\hat{\pi}_N(s, a) = \pi_N(s, a)$.

\subsection{V-MCTS Algorithm}

\begin{wrapfigure}{R}{0.6\textwidth}
\vskip -0.5cm
\begin{minipage}{0.6\textwidth}
\begin{algorithm}[H]
    \centering
    \caption{Virtual MCTS}\label{alg:stop}
    \begin{algorithmic}[1]
        \State \text{Input: budget $N$, state $s$, conservativeness $r$, error $\epsilon$}
        \State \text{Notice: \Changes{$\pi_k(s)$, $\hat{\pi}_k(s)$} are policy distributions.}
        \State \text{\Changes{Notice: $\pi_k(s, a)$, $\hat{\pi}_k(s, a)$ are probabilities for action $a$.}}
        \For{$k \in N$}
            \State \text{Selection with \Changes{UCB1}}
            \State \text{Expansion for the new node}
            \State \text{\color{red}{Evaluation}} \text{ with Neural Networks}
            \State \text{Backpropagation for updating Q and visitations}
            \State $\pi_k(s, a) \leftarrow N_k(s, a)/k$
            \State \text{Virtual expand $N - k$ nodes and update} $\hat{N}(s, a)$ \label{alg:line_5}
            \State $\Changes{\hat{\pi}_k(s, a)} \leftarrow \hat{N}_k(s, a)/N$
            \If {$k \ge rN \land \left|\left| \hat{\pi}_k(s) - \hat{\pi}_{k/2}(s) \right|\right|_{1} < \epsilon$}
            \State $\pi(s) \leftarrow \hat{\pi}_k(s)$
            \State \textbf{Break}
            \EndIf
            \State $\pi(s) \leftarrow \pi_k(s)$
        \EndFor
        \State \textbf{Return} $\pi(s)$
    \end{algorithmic}
\end{algorithm}
\end{minipage}
\end{wrapfigure}

The procedure of MCTS with VET-Rule is listed in Algorithm \ref{alg:stop}. 
We name our method \textbf{Virtual MCTS} (V-MCTS), a variant of MCTS with VET-Rule.
Compared with the original MCTS, lines 8-13 are the pseudo-code for the rule.
In each iteration, we do some calculations with little cost to judge whether the VET-Rule is satisfied. If it is, then the search process is terminated and returns the current virtual expanded policy $\hat{\pi}_k(s)$. Thus, it skips the next $N - k$ model predictions from neural networks in the evaluation stage highlighted in \Changes{line 7}. In this way, we can approximate the oracle distribution $\pi_N$ by $\hat{\pi}_k$ while reducing the budget of $N$ simulations to $k$. Here, $k \ge rN$ and $r$ is a hyperparameter of the minimum budget $rN$. We can reduce the tree size by $1/r$ times at most. 

\subsection{Theoretical Justifications}
Furthermore, we will give some theoretical bounds on the $Q$-values and $\hat{\Delta}_s(k, N)$ of V-MCTS. Before this, we define some notations first: $k$ is the index of the current search iteration, and $N$ is the number of total search iterations. $\mathcal{A}$ is the action set and $|\mathcal{A}| > 1$. 
Each action $a \in \mathcal{A}$ is associated with a value, which is a random variable bounded in the interval $[0, 1]$ with expectation $Q_a$. At step $k$, the empirical mean value over the $k$ trails is $Q_{k}(s, a)$. For simplification, we denote $Q_{k}(s, a)$ as $\Bar{Q}_a^k$. 
We denote the empirical mean value after $N - k$ virtual expansion as $\hat{Q}_a^N$.
Since we only deal with the visitation distribution of the root, we omit the state subscript for the root state. For convenience, we assume that different actions are ordered by their corresponding expected values, which means that $1 \ge Q_1 \ge Q_2 \ge \cdots \ge Q_a \ge \cdots \ge Q_{|\mathcal{A}|} \ge 0$.

\begin{theorem}
\label{theorem:th_2}
Given $r \in (0, 1)$, confidence $\delta \in (0, 1)$, finite action set $\mathcal{A}$.
$\exists N_0 > 0$, $\forall N > N_0, k \ge rN$, let $\epsilon_k = \sqrt{\frac{1}{2k} \ln{\frac{100k^2}{\delta}}}$, after $k$ times vanilla expansion and $N - k$ times virtual expansion, we have (a) \textbf{Value Consistency in Virtual Expansion}: $Pr\{\bigcap_{a \in \mathcal{A}} \left| \hat{Q}_a^N - Q_a \right| < \epsilon_k\}
        > (1 - \frac{e \delta |\mathcal{A}|}{50 r^2 N^2})$; (b) \textbf{Best Action Identification in Virtual Expansion}: $Pr\{ \left| \hat{Q}_*^N - \Bar{Q}_1^N \right| < \epsilon_k + \epsilon_N\}
        > 1 - 2 (\frac{\delta}{50 k^2} \exp{(\frac{1}{1.61 \sqrt{k}})} + \frac{\delta}{50 N^2} \exp{(\frac{1}{N})})$, where $* := \arg\max_{a \in \mathcal{A}} \Bar{Q}_a^k$, $e$ is the Euler's number.
\end{theorem}

Theorem \ref{theorem:th_2} (a) gives a bound of the distance between the empirical mean values after virtual expansions and the expected values. Noticed that $\lim_{N \rightarrow \infty} \epsilon_{rN} = 0$ and $\lim_{N \rightarrow \infty} \frac{e \delta |\mathcal{A}|}{50 r^2 N^2} = 0$. It tells that, after enough trails, the expected $Q$-values of all actions can be estimated by the corresponding empirical $Q$-values after virtual expansion. Furthermore, when the $Q$-values have converged, the effect of virtual expansion is the same as that of vanilla expansion.
Denote the best empirical action as $*$, and the best expected action is $1$ because $Q_1 \ge Q_a$. 
Theorem \ref{theorem:th_2} (b) notes that the $Q$-value of the best empirical action with virtual expansion is of high probability to be close to the Q-value of the best expected action with vanilla expansion.
Intuitively, 
it tells that whether or not we successfully find the best expected action, the best empirical action has similar effects to the best expected action. And, $N_0$ should be larger than the action space size, otherwise it cannot satisfy the theorem conditions. 
The proof is attached in Appendix.

\begin{theorem}
\label{theorem:th_4}
(\textbf{Error Bound of V-MCTS}): 
Given $r \in (0, 1)$, confidence $\delta \in (0, 1)$, finite action set $\mathcal{A}$. 
Suppose the virtual expanded policy $\hat{\pi}_k$ is generated from Algorithm \ref{alg:stop} (V-MCTS), 
$\exists N_0 > 0$, $\forall N > N_0, k \ge rN$, $\forall \epsilon  \in (0, 1]$,
if $\left|\left| \hat{\pi}_k(s) - \hat{\pi}_{k/2}(s) \right|\right|_{1} < \epsilon$, we have 
$Pr\{ \left|\left| \pi_N(s) - \hat{\pi}_{k}(s) \right|\right|_{1}  < 3 \epsilon \} > 1 - \frac{e \delta |\mathcal{A}|}{50 N^2} (1 + \frac{4}{r^2})$, where $e$ is the Euler's number.
\end{theorem}

Theorem \ref{theorem:th_4} tells that a bound of $\hat{\Delta}_{s}(k, k/2)$ implies a bound of $\hat{\Delta}_s(k, N)$ with high probability. Noticed that $\lim_{N \rightarrow \infty} \frac{e \delta |\mathcal{A}|}{50 N^2} (1 + \frac{4}{r^2}) = 0$. Therefore, the oracle policy $\pi_N(s)$ can be approximated by $\hat{\pi}_k(s)$ after enough trails.
The proof of this theorem is attached in Appendix.

Given the minimum distance $\epsilon$, for easier states, the rule $\hat{\Delta}_s(k, k/2) < \epsilon$ is easier to satisfy. That's because the Q-values of the tree nodes keep in a small range even with more search iterations.
Thus, the virtual expanded policy generated by V-MCTS is close to the oracle policy, and the search loop will be terminated earlier if the state is easier.
In the next section, we do ablations to investigate the effects of the hyper-parameters and show visualizations to verify the adaptive behavior. 

\section{Experiments}
In this section, the goal of the experiments is to prove the effectiveness and efficiency of V-MCTS. First, we compare the performance and the cost between the vanilla MCTS and our method. Specifically, we evaluate the board game Go $9 \times 9$, and a few Atari games. In addition, we do some ablations to examine the virtual expansion's effectiveness and evaluate the sensitiveness of hyper-parameters. 
Finally, we show the adaptive mechanism through visualizations and performance analysis.

\subsection{Setup}
\textbf{Models and Environments} Recently, Ye \textit{et al.} \cite{ye2021mastering} proposed EfficientZero, a variant of MuZero \citep{schrittwieser2020mastering} with three extra components to improve the sample efficiency, which only requires 8 GPUs in training, and thus it is more affordable. Here we choose the board game Go $9 \times 9$ and a few Atari games as our benchmark environments. The game of Go tests how the algorithm performs in a challenging planning problem. And Atari games feature visual complexity.

\textbf{Hyper-parameters} As for the Go $9 \times 9$, we choose Tromp-Taylor rules. The environment of Go is built based on an open-source codebase, GymGo \citep{gymgo}. We evaluate the performance of the agent against GNU Go v3.8 at level 10 \citep{gnugo} for 200 games. We include 100 games as the black player and 100 games as the white one with different seeds. We set the komi to 6.5, as most papers do.
As for the Atari games, we choose 5 games with 100k environment steps. In each setting, we use 3 training seeds and 100 evaluation seeds for each trained model. More details are attached in Appendix.

\textbf{Baselines} We compare our method to EfficientZero with vanilla MCTS, on Go $9 \times 9$ and some Atari games. Moreover, DS-MCTS \citep{lan2020learning} also terminates the MCTS adaptively through trained uncertainty networks. But it requires specific features designed for Go games, and it only works in the evaluation stage. Therefore, we also compare the final performance for Go games with the DS-MCTS.

\subsection{Results on Go}
Figure \ref{fig:performance} illustrates the computation and performance trade-off on Go against the same GnuGo (level 10) agent. The x-axis is the training speed, and the y-axis is the winning rate. Therefore, the curve which lies to the top-left has better performance than the bottom-right in terms of the trade-off. We train the baseline method with constant budgets $N$, which is noted as the blue points. Besides, we also train the V-MCTS with hyperparameters $r=0.2, \epsilon=0.1$. We evaluate the trained model with different $\epsilon$ to display the trade-off between computation and performance, indicated as the red points. And the green points are the GnuGo with different levels. The GnuGo engine provides models of different levels (1-10). Each level is a trade-off between the run time and the strength of the agent. The y-axis is the winning rate against the model of level 10. Here the green curve shows the performance-computation trade-off of the GnuGo engine.

\begin{figure*}[t]
\vspace{-2em}
	\centering
	\subfigure[Evaluations of Performance]{
			\label{fig:performance}
			\centering
			\includegraphics[width=2.6in]{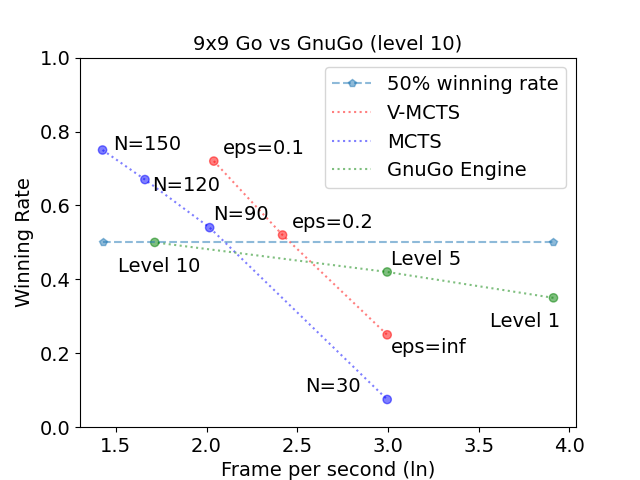}
	}%
	\subfigure[Wining Rates and Tree Size during Training]{ 
		    \label{fig:training}
			\centering
			\includegraphics[width=2.6in]{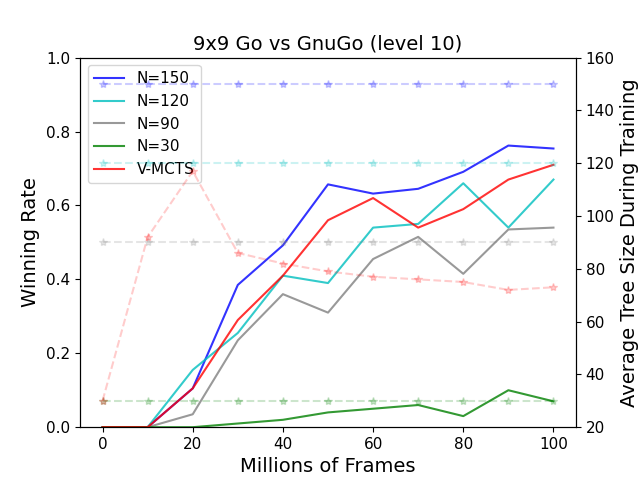}
	}%
	\centering
	\caption{
    Performance of Virtual MCTS on Go $9 \times 9$ against GnuGo (level 10). (a) Evaluating the speed and winning rate of MCTS, V-MCTS, and GnuGo at different levels. 
    V-MCTS has better computation and performance trade-off. X-axis is the frame per second in the ln scale. \Changes{(For convenience, the eps in Figure (a) denotes the hyperparameter $\epsilon$.)} 
    (b) Evaluating the winning rate and the average tree size during the training stage. The \textbf{solid lines} and \textbf{dashed lines} display the winning rate and the tree size, respectively.
    The red one is V-MCTS, and the others are vanilla MCTS with different $N$.
    V-MCTS makes the tree size adaptive in training and reduces the search cost while performing comparably to the vanilla MCTS ($N=150$). However, reducing $N$ in vanilla MCTS results in a more significant performance drop.
    }
	\label{fig:res}
	\vspace{-1em}
\end{figure*}

Firstly, for all the methods, more search iterations (larger $N$ or smaller $eps$) lead to higher winning rates but result in more response time. 
Secondly, V-MCTS \Changes{($\epsilon=0.1$)} achieves \textbf{71\%} winning rate against the GnuGo level 10, which is close to 75\% from MCTS ($N=150$). And the time cost of V-MCTS \Changes{($\epsilon=0.1$)} for a one-step move is \textbf{0.12s} while the GnuGo engine is 0.18s and MCTS ($N=150$) is more than 0.2s. Therefore, such termination rule can keep strong performances with less budget. For a more detailed breakdown of the time consumption of V-MCTS, please see the Appendix. 
Finally, we can find that the red dashed line lies to the right of the blue one. It indicates that V-MCTS is better than the vanilla MCTS considering the computation and performance trade-off.

Figure \ref{fig:training} illustrates the changes of winning rates and the average tree size over the training stage. Here, as the red dashed line shows, the tree size of V-MCTS varies over training and keeps smaller than the maximum size ($N=150$) while the winning rate keeps comparable to the MCTS ($N=150$). Consequently, V-MCTS can work well not only in evaluation but also in training.

\begin{wraptable}{R}{0.7\textwidth}
\caption{Results for Go $9 \times 9$: Comparison of the winning rate and the average budget over 200 games \Changes{for 3 separate training runs}.
}
\label{table:ds_mcts}
\begin{center}
\begin{tabular}{l|l|ll}
\hline
 & MCTS (N=150) & DS-MCTS & V-MCTS \\
\hline
Average budget & 150 $\pm$ 0.0 & 97 $\pm$ 12.5 & \textbf{76} $\pm$ 10.8 \\
Winning rate & \textbf{75\%} $\pm$ 3.0\% & 60\% $\pm$ 4.0\% & 71\% $\pm$ 4.7\% \\
\hline
\end{tabular}
\end{center}
\end{wraptable}

\begin{table*}[b]
\vspace{-1em}
\caption{Results for Atari games: scores over 100 evaluation seeds \Changes{for 3 separate training runs}. $k$ is the average budget of V-MCTS. MCTS ($N = 50$) is the oracle one. The best results among distinct versions except the oracle are in bold. V-MCTS achieves better performance-computation trade-off.}
\label{table:atari}
\small
\begin{center}
\begin{tabular}{ll|lll|l}
\hline
\bf MCTS   & \bf $N = 50$ & \bf $N = 30$ & \bf $N = 10$ & \bf Ours (V-MCTS) & \textbf{Budget} $k$ \\
\hline 
Pong & 19.7 $\pm$ 1.6 & 12.5 $\pm$ 5.5 & 2.0 $\pm$ 1.3 & \textbf{18.8} $\pm$ 2.8 & 13.3 $\pm$ 0.6 \\
Breakout & 410.7 $\pm$ 15.1 & 370.9 $\pm$ 34.1 & 303.9 $\pm$ 11.3 & \textbf{372.8} $\pm$ 18.3 & 15.7 $\pm$ 0.6 \\
Seaquest & 1159.9 $\pm$ 90.7 & 775.2 $\pm$ 146.8 & 555.4 $\pm$ 66.9 & \textbf{970.0} $\pm$ 339.5 & 14.3 $\pm$ 1.2 \\
Hero & 9992.1 $\pm$ 2059.4 & \textbf{9241.3} $\pm$ 3615.3 & 4437.0 $\pm$ 2490.6 & 8928.1 $\pm$ 2922.1 & 15.0 $\pm$ 1.0 \\
Qbert & 14495.8 $\pm$ 683.9 & 10429.9 $\pm$ 2291.1 & 8149.8 $\pm$ 2085.0 & \textbf{11476.6} $\pm$ 978.2 & 16.3 $\pm$ 1.2 \\
\hline
\end{tabular}
\end{center}
\end{table*}


We also compare our method to DS-MCTS~\citep{lan2020learning}, which terminates the search when the state is predicted to be certain with DNNs. To make fair comparisons, we implement the DS-MCTS and follow their design of features for Go games. We set $N_{max}=150, c=\{30, 75, 120\}, thr=\{.1, .1, .1\}$ in DS-MCTS.
Then we compare the winning rate and the average budget among the vanilla MCTS, DS-MCTS, and V-MCTS. Experiments show that V-MCTS outperforms the DS-MCTS in both aspects, listed in Table \ref{table:ds_mcts}. We attribute the better performance of V-MCTS to the virtual expanded policy. DS-MCTS chooses $\pi_k(s)$ as the policy after the termination of the search, while V-MCTS chooses $\hat{\pi}_k(s)$, which has theoretical guarantees to approximate $\pi_N(s)$.

\subsection{Results on Atari}

Apart from the results of Go, we also evaluate our method on some visually complex games. Since the search space of Atari games is much smaller than that of Go and the Atari games are easier, we choose a few Atari games to study how the proposed method impacts the performance. 
We follow the setting of EffcientZero, 100k Atari benchmark, which contains only 400k frames data. The results are shown in Table \ref{table:atari}. Generally, we find that our method works on Atari games. The tree size is adaptive, and the performance of V-MCTS is still comparable to the MCTS with full search trails. It has better performance than the MCTS($N=30$) while requiring fewer searches, proving the effectiveness and efficiency of our proposed method. The Hero game is an outlier here. But our performance is very close to the vanilla MCTS ($N=30$) while we use half of the search iterations on average.
Besides, the number of search times decreases more than that on Go. 

To sum up, V-MCTS can keep comparable performance under fewer search iterations while simply reducing the total budget of MCTS will encounter a more significant performance drop. In addition, the savings of search cost is more substantial in easier environments. 

\subsection{Ablation Study}

The results in the previous section suggest that our method reduces the response time of MCTS while keeping comparable performance on challenging tasks. This section tries to figure out which component contributes to the performance and how the hyperparameters affect it. And we also ablate the effects of different normalization criterions in VET-Rule and the larger budget ($N$) in MCTS.

\textbf{Virtual Expansion} In Section \ref{sec:virtual_expansion}, we introduce the virtual expansion.
To prove the effectiveness of virtual expansion, we compare it with another two baseline expansion methods. One is the vanilla expansion, mentioned in Algorithm \ref{alg:expand}, which returns at iteration $k$ and outputs $\pi_k$.
Another is greedy expansion, which spends the left $N - k$ simulations in searching the current best action greedily, indicating that $\hat{\pi}_k(s, a) = (N_k(s, a) + (N - k) \mathbf{1}_{b=\arg\max N_k(s, b)}) / N$.
Briefly, we stop the search process after $k=30$ iterations and do $N-k$ times virtual expansion or greedy expansion or nothing, where $k = rN$ and $r=0.2, N=150$.

\begin{wraptable}{R}{0.5\textwidth}
\caption{Ablation results of different expansion methods on Go $9 \times 9$ \Changes{for 3 separate training runs}.}
\label{ablation}
\begin{center}
\begin{tabular}{lll}
\hline
\bf Algorithm  & \bf Size Avg. & \bf Winning Rate \\
\hline 
Vanilla expansion & 30 & 17\% $\pm$ 3.2\% \\
Greedy expansion & 30 & 3\% $\pm$ 2.0\% \\
Virtual expansion & 30 & \textbf{32\%} $\pm$ 3.5\% \\
\hline
\end{tabular}
\end{center}
\end{wraptable}

\begin{figure*}[t]
\vspace{-2em}
	\centering
	\subfigure[Evaluations of Performance]{
			\label{fig:abl_r}
			\centering
			\includegraphics[width=2.6in]{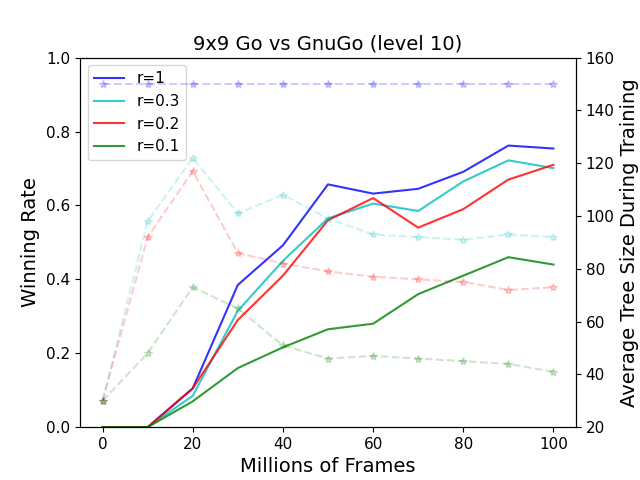}
	}%
	\subfigure[Wining Rates and Tree Size during Training Stage]{ 
		    \label{fig:abl_eps}
			\centering
			\includegraphics[width=2.6in]{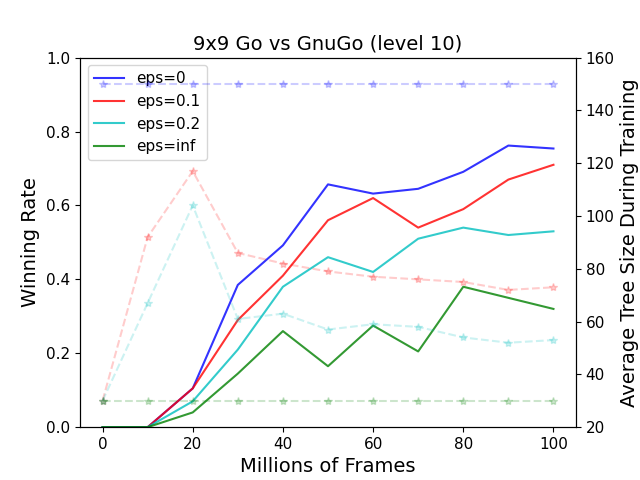}
	}%
	\centering
	\caption{
    Sensitivity of VET-Rule to the hyperparameter $r, \epsilon$ on Go $9 \times 9$. The solid lines and dashed lines display the winning probability and the average tree size, respectively.
    }
	\label{fig:ablation}
\end{figure*}

We compare the winning rate against the same engine, and the results are listed as Table \ref{ablation} shows. 
The winning rate of virtual expansion can achieve 32\%, which is much better than the others. Besides, MCTS with greedy expansion does not work because it over-exploits and results in severe exploration issues. Consequently, virtual expansion can generate a better policy distribution because it can balance exploration and exploitation with UCT.

\textbf{Termination Rule}
It is significant to explore a better termination rule to keep the sound performance while decreasing the tree size as much as possible. As mentioned in Section \ref{sec:termination}, VET-Rule has two hyperparameters $r, \epsilon$. Here $r$ is the factor of the minimum budget $rN$, and $\epsilon$ is the minimum distance $\hat{\Delta}_s(k, k/2)$.
To explore the VET-Rule with better computation and performance trade-off, we do ablations for the different values of $r$ and $\epsilon$, respectively. The default values of $r, \epsilon$ are set to $0.2, 0.1$.

Figure \ref{fig:ablation} compares the winning rate as well as the average tree size across the training stage.
Firstly, Figure \ref{fig:abl_r} gives the results of different minimum search times factor $r$.
The winning probability is not sensitive to $r$ when $r \ge 0.2$.
Nevertheless, the average tree size is sensitive to $r$ because V-MCTS is supposed to search for at least $rN$ times.
In addition, there is a performance drop between $r=0.1$ and $r=0.2$. Therefore, it is reasonable to choose $r=0.2$ to balance the speed and the performance.

Besides, the comparisons of the different minimum distance $\epsilon$ are shown in Figure \ref{fig:abl_eps}. A larger $\epsilon$ makes the tree size smaller because $\hat{\Delta}_s(k, k/2) < \epsilon$ is easier to satisfy. In practice, the performance is highly correlated with $\epsilon$. In terms of the winning rate, a smaller $\epsilon$ outperforms a larger one. However, better performances are at the cost of more computations. 
We suggest selecting an appropriate minimum distance to balance the computation and performance ($r=0.2, \epsilon=0.1$).

\Changes{\textbf{Normalization criterion in VET-Rule} The proposed VET-Rule, $\left|\left| \hat{\pi}_k(s) - \hat{\pi}_{k/2}(s) \right|\right| < \epsilon$ is a termination condition for V-MCTS. And L2 norm is another reasonable choice to amplify the bigger deviations. Therefore, we make ablations of the normalization criterion for the policy distributions. Specifically, we take a pretrained model, and compare the different strategies of L1 norm and L2 norm, namely, $\left|\left| \hat{\pi}_k(s) - \hat{\pi}_{k/2}(s) \right|\right|_{1} < \epsilon$ and $\left|\left| \hat{\pi}_k(s) - \hat{\pi}_{k/2}(s) \right|\right|_{2} < \epsilon$. The results are as Tab. \ref{table:vet_rule} shows. We can find that (1) L2 norm can also work for V-MCTS; (2) L1 norm is better than L2 norm. And we attribute this to the formulation of ucb scores. Because the ucb scores have already taken into account the difference in the visitations (see the N(s, a) in Eq (1)). Therefore, amplifying the deviations may result in some bias. }

\begin{table}[!h]
\caption{\Changes{Comparison of the winning rate and the average budget with different norm strategies in VET-Rule. L1 Norm means $\left|\left| \hat{\pi}_k(s) - \hat{\pi}_{k/2}(s) \right|\right|_{1} < \epsilon$ and L2 Norm means $\left|\left| \hat{\pi}_k(s) - \hat{\pi}_{k/2}(s) \right|\right|_{2} < \epsilon$.}}
\label{table:vet_rule}
\begin{center}
\begin{tabular}{l|ll}
\hline
\bf  & Average budget & Winning rate \\
\hline  
MCTS ($N=150$) & $150$ & $82.0\%$ \\
\hline
V-MCTS \textbf{L1} Norm, $N=150, r=0.2, \epsilon=0.1$ & $\textbf{96.2}$ & $\textbf{81.5\%}$ \\
V-MCTS \textbf{L2} Norm, $N=150, r=0.2, \epsilon=0.1$ & $97.1$ & $79.8\%$ \\
V-MCTS \textbf{L2} Norm, $N=150, r=0.2, \epsilon=0.05$ & $119.3$ & $81.0\%$ \\
\hline
\end{tabular}
\end{center}
\end{table}

\Changes{\textbf{Larger budget ($N$) in MCTS} To investigate whether our method still holds with larger amounts of MCTS expansions, we take a pretrained model and compare two strategies: (1) vanilla expansion with N=150/400/600/800 nodes in MCTS (2) virtual expanded policy with $N=800, r=0.2, \epsilon=0.1$. The results are listed in Tab. \ref{table:more_budget}. The result shows that (1) V-MCTS($N=800, r=0.2, \epsilon=0.1$) is better than MCTS ($N=600$) in both the average budget and the winning rate, (2) V-MCTS can achieve comparable performance to the oracle MCTS($N=800$) while keeping much less average budget. Therefore, V-MCTS works with a larger amount of MCTS expansions.}

\begin{table}[!h]
\caption{\Changes{Comparison of the winning rate and the average budget with larger amounts of MCTS expansions. Here the hyper-parameters of our method are $N=800, r=0.2, \epsilon=0.1$.}}
\label{table:more_budget}
\begin{center}
\begin{tabular}{l|llll|l}
\hline
\bf MCTS & $N=150$ & $N=400$ & $N=600$ & $N=800$ & Ours \\
\hline  
Average budget & $150$ & $400$ & $600$ & $800$ & $431.1$ \\
Winning rate & $82.0\%$ & $84.5\%$ & $84.9\%$ & $85.9\%$ & $85.0\%$ \\
\hline
\end{tabular}
\end{center}
\end{table}

\subsection{Visualization of V-MCTS's Adaptive Behavior}
\label{sec:vis}

\begin{figure*}[t]
\vspace{-2em}
	\centering
	\subfigure[Play as Black]{
			\label{fig:abl_r}
			\centering
			\includegraphics[width=13cm]{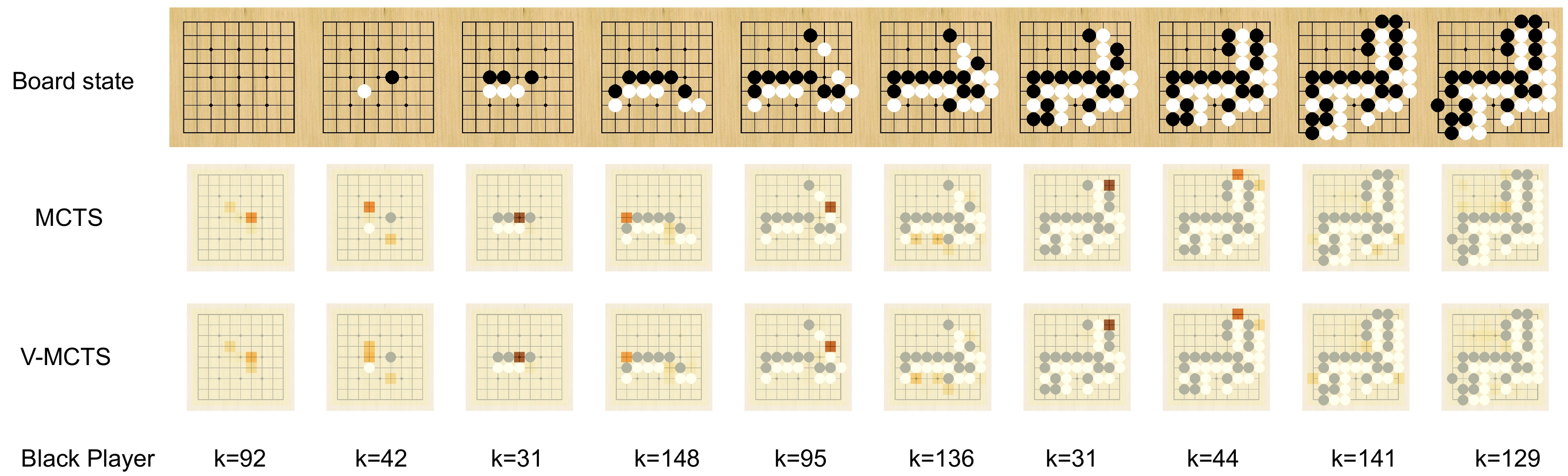}
	}%
	\\
	\subfigure[Play as White]{ 
		    \label{fig:abl_eps}
			\centering
			\includegraphics[width=13cm]{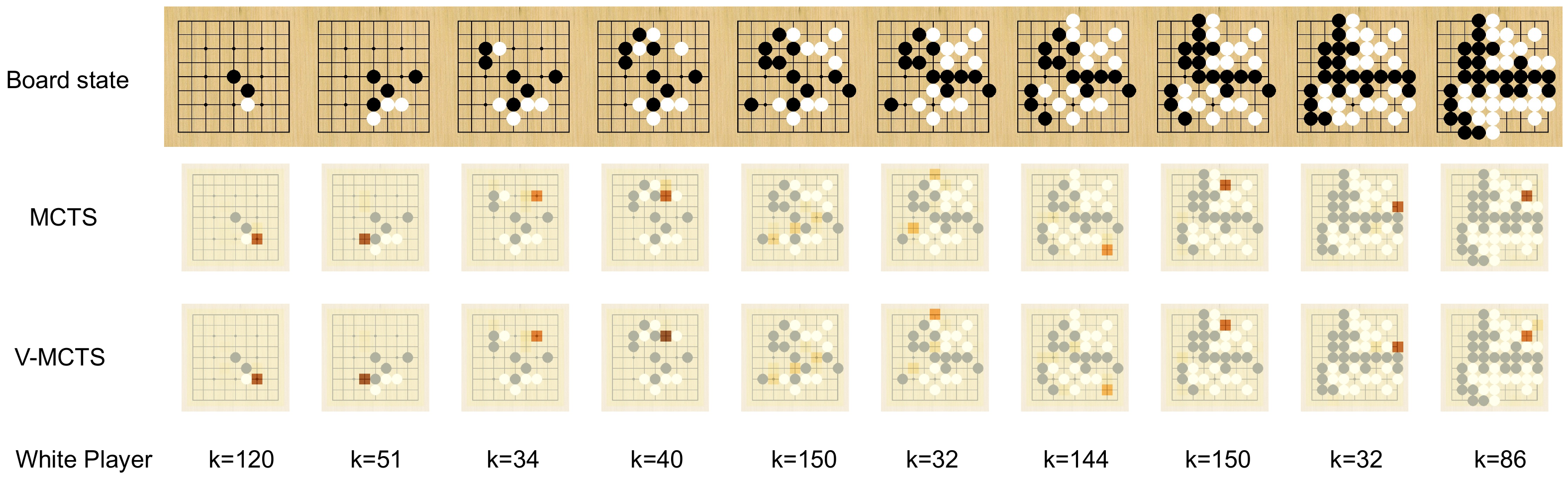}
	}%
	\centering
	\caption{
	\Changes{
	Heatmap of policy distributions from the MCTS ($N=150$) and the V-MCTS. The agent play as Black in (a) and White in (b) against the GnuGo (level 10). Our agent wins in both of the games. A darker red color represents larger visitations of the corresponding action. The V-MCTS will terminate with different search times $k$ according to the situations and generate a near-oracle policy distribution.
	}
    }
	\label{fig:heatmap}
	\vspace{-1.5em}
\end{figure*}

In this section, we display the adaptive mechanism through visualizations and performance analysis. We find that (1) $\hat{\pi}_k(s)$ is close to $\pi_N(s)$; (2) V-MCTS terminates earlier for simpler states.

Specifically, we choose some states at different time steps on one game of Go against the GnuGo with a trained model. And Figure \ref{fig:heatmap} is the visualization of the policy distributions heatmap.
Here we add two games, which contain one black player and one white player. The last two rows in each subfigure are the heatmap visualization for oracle MCTS ($\pi_N$) and V-MCTS ($\hat{\pi}_k$ when $\hat{\Delta}_s(k, k/2) < \epsilon$). The darker the color is on the grid, the more the corresponding action is visited during the search. In general, $\hat{\pi}_k$ is close to the $\pi_N$ at distinct states, indicating that the virtual expanded policy obtained after virtual expansion is close to the oracle one. 

Furthermore, the less valuable actions there are, the sooner the V-MCTS will terminate. For example, on Go games, the start states are usually not complex because there are only a few stones on the board, but the situations are more complicated in shuban, the closing stage of the game. Notably, the termination occurs earlier in the start states (columns 1, 2, 3), but it is the opposite when the situation is more complicated.
More importantly, the termination step $k$ is not related to the number of Go pieces. 
Therefore, we can conclude that V-MCTS makes adaptive terminations according to the situations of the current states and generate near-oracle policies. Specifically, it terminates the search loop earlier when handling easier states, which has a better computation and performance trade-off. 

\section{Discussion}
This paper proposes a novel method named V-MCTS to accelerate the MCTS to determine the termination of search iterations. It can maintain comparable performances while reducing half of the time to search adaptively. We believe that this work can be one step toward applying the MCTS-based methods to some real-time domains. One limitation of our work is that it cannot deal with the environments of continuous action space. In the future, we will plan to extend to the continuous action space with early termination.

\begin{ack}
This work is supported by the Ministry of Science and Technology of the People´s Republic of China, the 2030 Innovation Megaprojects "Program on New Generation Artificial Intelligence" (Grant No. 2021AAA0150000). This work is also supported by a grant from the Guoqiang Institute, Tsinghua University.
\end{ack}

\bibliographystyle{abbrv}
\bibliography{main}

\newpage

\appendix
\section{Appendix}

\subsection{Experimental setup}
\label{app:setting}

\subsubsection{Models and Hyper-parameters}

\textbf{MCTS in modern RL algorithms} As mentioned in the related works, modern MCTS-based RL algorithms include four stages in the search loop, namely selection, expansion, evaluation, and backpropagation. 
(1) The selection stage targets selecting a new leaf node with UCT. (2) The expansion stage expands the selected node and updates the search tree. (3) The evaluation stage evaluates the value of the new node. (4) The backpropagation stage propagates the newly computed value to the nodes along the search path to obtain more accurate Q-values with Bellman backup. 

\textbf{Model Design}
As for the architecture of the networks, we follow the implementation of EfficientZero \citep{ye2021mastering} in Atari games, which proposes three components based on MuZero: self-supervised consistency, value prefix, and off-policy correction.
In the implementation of EfficientZero, there is a representation network, a dynamics network, and a reward/value/policy prediction network. The representation network is to encode observations to hidden states. The dynamics network is to predict the next hidden state given the current hidden state and an action. The reward/value/policy prediction network is to predict the reward/value/policy. Notably, they propose to keep temporal consistency between $s_{t+1}$ and the predicted state $\hat{s}_{t+1}$. The training objective is:
\begin{equation}
    \begin{aligned}
        \mathcal{L}_t(\theta) &= \lambda_1 \mathcal{L}(u_t, r_t) + \lambda_2 \mathcal{L}(\pi_t, p_t) + \lambda_3 \mathcal{L}(z_t, v_t)  + \lambda_4 \mathcal{L}_{\text{similarity}}(s_{t+1}, \hat{s}_{t+1}) + c ||\theta||^2 \\
        \mathcal{L}(\theta) &= \frac{1}{l_{\text{unroll}}} \sum_{i=0}^{l_{\text{unroll}}-1} \mathcal{L}_{t+i}(\theta),
    \end{aligned}
\end{equation}
where $u_t, \pi_t, z_t$ are the target reward/policy/value of the state $s_t$ and $r_t, p_t, v_t$ are the predicted reward/policy/value of the state $s_t$ respectively. The prediction will do $l_{\text{unroll}}=5$ times iteratively for the state $s$ on both Go and Atari games. We do some changes when dealing with board games. Significantly, we remove the reward prediction network because the agent will receive a reward only at the end of the games. The other major changes for board games are listed as follows.

Since the board game Go is harder than the Atari games, we add more residual blocks (two times blocks). Specifically, we use 2 residual blocks in the representation network, the dynamics network, as well as the value/policy prediction network on Go $9 \times 9$ while EfficientZero uses only 1 residual block in those networks on Atari games.
As for the representation network, we remove the downsampling part here because there is no need to do downsampling for Go states.
In the value/policy prediction networks, we enlarge the dimension of the hidden layer from 32 to 128.
Besides, considering that the reward is sparse on Go (only the final value) and the collected data are sufficient, we only take the self-supervised consistency component in EfficientZero to give more temporal supervision during training.

\textbf{Hyper-parameters}
In each case, we train EfficientZero for unrolled 5 steps and mini-batches of size 256. Besides, the model is trained for 100k batches with 100M frames of data on board games while 100k batches with 400k frames in Atari games. We stack 8 frames in board games without frameskip while stacking 4 frames in Atari games with a frameskip of 4. During both training and evaluation, EfficientZero chooses 150 simulations for each search in board games while 50 simulations of budget for Atari games. Other hyper-parameters are listed in Table \ref{tab:param}.

\begin{table}[!h]
    \centering
    \begin{tabular}{c|c|c}
    \hline
         & \bf Go $9 \times 9$ & \bf Atari \\
    \hline
        Maximum number of tree size & 150 & 50 \\
        Observation down-sampling & No & 96 $\times$ 96 \\
        Total frames & 100M & 400k \\
        Replay buffer size & 2M & 100k \\
        Max frames per episode & 163 & 108k \\
        Cost of training time & 24h & 8h \\
        Komi of Go & 6.5 & - \\
        Frame stack & 8 & 4 \\
        Frame skip & 1 & 4 \\
        Training steps & 100k & 100k \\
        Mini batch & 256 & 256 \\
        Learning rate & 0.05 & 0.2 \\
        Weight decay ($c$) & 0.0001 & 0.0001 \\
        Reward loss coefficient ($\lambda_1$) & 0 & 1 \\
        Policy loss coefficient ($\lambda_2$) & 1 & 1 \\
        Value loss coefficient ($\lambda_3$) & 1 & 0.25 \\
        Consistency loss coefficient ($\lambda_4$) & 2 & 0.5 \\
        Dirichlet $\alpha$ & 0.03 & 0.3 \\
        $c_1$ in P-UCT & 1.25 & 1.25 \\
        $c_2$ in P-UCT & 19652 & 19652 \\
        $\epsilon$ & 0.1 & 0.1 \\
        $r$ & 0.2 & 0.2 \\
    \hline
    \end{tabular}
    \caption{Hyper-parameters of V-MCTS on Go $9 \times 9$ and Atari games}
    \label{tab:param}
\end{table}

\subsubsection{Training Details of Go}
\label{app:go}
The detailed implementations of Atari games are discussed in EfficientZero \citep{ye2021mastering}. However, it is nontrivial to adapt to board games. Here we give detailed instructions for training board games Go $9 \times 9$ in our implementations.

\textbf{Inputs} We follow the designs of AlphaZero, and we use the Tromp-Taylor rules, which is similar to previous work \citep{silver2018general, schrittwieser2020mastering}. The input states of the Go board are encoded into a $17 \times 9 \times 9$ array, which stacks the historical 8 frames and uses the last channels $C$ to identify the current player, 0 for black and 1 for white. Notably, the one historical frame consists of two planes $[X, Y]$, where the first plane $X$ represents the stones of the current player and the second one $[Y]$ represents the stones of the opponent. Besides, if there is a stone on board, then the state of the corresponding position in the frame will be set to 1, otherwise to 0. 
For example, if the current player is black and suppose $b[i, j]$ is the current board state, $X[i,j]=\mathbf{1}_{b[i, j] = \text{black stone}}, Y[i,j]=\mathbf{1}_{b[i, j] = \text{white stone}}$.
In summary, we concatenate together the historical planes to generate the input state $s=[X_{t-7}, Y_{t-7}, X_{t-6}, X_{t-6}, ..., X_t, Y_t, C]$, where $X_t, Y_t$ are the feature planes at time step $t$ and $C$ gives information of the current player.

\textbf{Training}
As for the training phase, we train the model from scratch without any human expert data, which is the same as the setting of Atari games. Besides, limited to the GPU resources, we do not use the reanalyzing mechanism of MuZero \citep{schrittwieser2020mastering} and EfficientZero \citep{ye2021mastering}, which targets at recalculation of the target values and policies from trajectories in the replay buffer with the current fresher model. Specifically, we use 6 GPUs for doing self-play to collect data, 1 GPU for training, and 1 GPU for evaluation.

\textbf{Exploration} 
To make a better exploration on Go, we reduce the $\alpha$ in the Dirichlet noise Dir$(\alpha)$ from 0.3 to 0.03, and we scale the exploration noise through the typical number of legal actions, which follows these works \citep{silver2018general, schrittwieser2020mastering}.
In terms of sampling actions from MCTS visit distributions, we will mask the MCTS visit distributions with the legal actions and sample an action $a_t$, where
\begin{equation}
    a_t :=
    \left\{
    \begin{aligned}
        a_t &\sim \pi_t, &t < T \\
        a_t &= \arg\max \pi_t, &t \ge T
    \end{aligned}
    \right. 
\end{equation}
$T$ is set to $16$ in self-play and is set to $0$ in evaluation, which is similar to these works \citep{silver2018general, schrittwieser2020mastering}. In this way, the agent does more explorations for the previous $T$ steps while taking the best action afterwards. But for Atari games, $\forall t$, we choose $a_t \sim \pi_t$ in self-play and $a_t = \arg\max \pi_t$ in evaluation, which is the same as these works \citep{schrittwieser2020mastering, ye2021mastering}.

\textbf{Two-player MCTS} 
On board games, there are two players against each other, which is different from that of one-player games. Therefore, we should do some changes to the MCTS with the two-player game. For one thing, the value network always predicts the Q-value of the black player instead of the current player, which provides a more stable prediction. Furthermore, A significant change is that during backpropagation of MCTS, the value should be updated with the negative value from the child node. Because the child node is the opponent, the higher value of the opponent indicates a worse value of the current player. 
Besides, as for $Q$-values of the unvisited children on Go and Atari games, we follow the implementation of EfficientZero \citep{ye2021mastering} as follows:
\begin{equation}
    \begin{aligned}
        \Bar{Q}(s^{\text{root}}) &= 0 \\
        \Bar{Q}(s) &= \frac{\Bar{Q}(s^{\text{parent}}) + \sum_{b}\mathbf{1}_{N(s,b) > 0}Q(s, b)}{1 + \sum_{b}\mathbf{1}_{N(s,b) > 0}} \\
        Q(s, a) :&=
        \left\{
            \begin{array}{rcl}
            Q(s, a)     &      & {N(s, a) > 0}\\
            \Bar{Q}(s)       &      & {N(s, a) = 0}
            \end{array} 
        \right. 
    \end{aligned}
\end{equation}
Notably, we allow the resignation for players when $\max_{a \in \mathcal{A}} Q(s^{root}, a) < -0.9$ during self-play and evaluation, which means that the predicted winning probability is less than $5\%$. For convenience, when playing against GnuGo during evaluation, our agent will follow the skip action if GnuGo agent chooses the skip action. As for other hyperparameters on both Go and Atari games, we note that we choose the same values as those in EffcientZero. Specifically, the $c_1, c_2$ in our mentioned P-UCT formula (Eq. \ref{eq:uct}) are set to 1.25 and 19652, following these works \citep{schrittwieser2020mastering, ye2021mastering}.


\subsubsection{Comparison of time cost on Go}

To give the comparison of time cost among the methods considering the languages and hardwares. Here, we list the detailed settings of our models and the GnuGo engines in Table \ref{table:hardware}. 

\begin{table}[!h]
\caption{Comparisons about the languages and hardware.}
\label{table:hardware}
\begin{center}
\begin{tabular}{l|lllll|l}
\hline
 & C & Python & CPU & GPU & Time \\
\hline
MCTS & \checkmark & \checkmark & \checkmark & \checkmark & 0.24 \\
V-MCTS & \checkmark & \checkmark & \checkmark & \checkmark & 0.12 \\
GnuGo & \checkmark & & \checkmark & & 0.18\\
\hline
\end{tabular}
\end{center}
\end{table}

\begin{table}[!h]
\caption{Extra time consumed by virtual expansion on Go. Here $k$ is vanilla expansion times and $T$ is extra virtual expansion times.}
\label{table:time_consumed}
\begin{center}
\begin{tabular}{l|llll}
\hline
\bf   & \bf $T=30$ & \bf $T=60$ & \bf $T=90$ & $T=120$ \\
\hline  
$k=30$ & 0.7ms & 1.5ms & 2.2ms & 3.0ms \\
\hline
\end{tabular}
\end{center}
\end{table}

To Give clear statistics about the extra time consumed by virtual expansion. Here, we record the time cost among different virtual expansion times in total after a fixed number of vanilla expansions. The results are listed in Table \ref{table:time_consumed}. We can find the extra time consumed by virtual expansion is little and linearly increased because there are only some atomic computations written in C++.

\subsection{Proof}
\label{app:proof}

Before the proof, let us recap some notations. In the MCTS procedure mentioned above, we suppose that there are total $|\mathcal{A}|$ actions to select and $N$ trials in total. Since the policy is defined as the visitation distributions of the root node, we only care about the $Q$-value and visitation changes of the root nodes. 

Each action $a \in \mathcal{A}$ is associated with a value, which is a random variable bounded in the interval $[0, 1]$ with expectation $Q_a$.
For convenience, we assume that different actions (arms) are ordered by their corresponding expected values, which means that $1 \ge Q_1 \ge Q_2 \ge \cdots Q_a \ge \cdots \ge Q_{|\mathcal{A}|} \ge 0$.
At $k$ iteration step of the search loop, the agent will select an action at the root note and receive an independent sample of its value $R_a^k \in [0, 1]$ from the neural networks, for simplification.
And at $k$ step, the action $a \in \mathcal{A}$ of the root node is selected for $N_k(s, a) = T_a^k \le k$ times of vanilla expansion.
We use the notation $\Bar{Q}_a^k = \frac{1}{T_a^k} \sum_{t=1}^{T_a^k} R_a^t$ to denote the empirical mean values and $U^k(s, a)$ to denote the ucb scores of the action $a$ given the root state $s$ at step $k$. (Here all $R^t_a$ are independent and bounded in $[0, 1]$). For virtual expansion, we use $\hat{Q}_a^k$ to denote the empirical mean values at step $k$. It is obvious that $\hat{Q}_a^N = \hat{Q}_a^k$ when $k$ satisfies the VET-rule.

\begin{lemma}
\label{lemma:1}
Given $r \in (0, 1)$, action set $\mathcal{A}$, $N > |\mathcal{A}|$. $\exists N_0, \forall N > N_0, k \ge rN$, we have $\forall a, T_a^k \ge 1$.
\end{lemma}
\begin{proof}
The ucb score is defined by Eq. (\ref{eq:uct}). Empirically, we will set $c_2$ given a budget $N$ and usually we have $c_2 > N \ge \sum_b N(s, b)$, so we note $c_1 + \log \frac{\sum_b N(s, b)+c_2+1}{c_2}$ as $c = c_1 + \log \frac{N+c_2+1}{c_2} \in (c_1, c_1 + \log 3)$
At step $k$, suppose there exist an action $a$, which has $T_a^{k} = 0$. Then at step $k$, the ucb score of $a$ should be
\begin{equation}
    \begin{aligned}
        U^k(s, a) 
        &= \Bar{Q}(s) + P(s,a)\frac{\sqrt{\sum_b N(s, b)}}{1+N(s,a)} c \\
        &> \Bar{Q}(s) + c_1M_a \sqrt{k} \\
        &> c_1M_a \sqrt{k}, \text{where } T_a^{k} = 0.
    \end{aligned}
\end{equation}
For action $b$, which has $T_b^{k} \ge 1$. At step $k$, we have
\begin{equation}
    \begin{aligned}
        U^k(s, b) 
        &= \Bar{Q}_b^k + P(s,b)\frac{\sqrt{\sum_i N(s, i)}}{1+N(s,a)}c \\
        &< 1 + (c_1 + \log 3)M_b \frac{\sqrt{k}}{1 + T_b^{k}} \\
        &< 1 + (c_1 + \log 3)M_b \frac{\sqrt{k}}{2}, \text{where } T_b^{k} \ge 1.
    \end{aligned}
\end{equation}
Since $k \ge rN > rN - |\mathcal{A}|$ and $f(k) = c_1M_a \sqrt{k} - (1 + (c_1 + \log 3)M_b \frac{\sqrt{k}}{2})$ is increasing for $k$. $\exists N_0$, we have $c_1M_a \sqrt{rN_0} = 1 + (c_1 + \log 3)M_b \frac{\sqrt{rN_0}}{2}$. Let $N_0 = \max\{ N_0, |\mathcal{A}| \}$,
$\forall N > N_0, f(k) > 0$, $c_1M_a \sqrt{k} > 1 + (c_1 + \log 3)M_b \frac{\sqrt{k}}{2}$. Then we have $U^k(s, a) > U^k(s, b)$. Therefore, at step $k$, for the action $b$ will be not selected. After extra $|\mathcal{A}|$ steps at most, all the action will be selected. Thus, we have $\forall a, T_a^k \ge 1$.
\end{proof}

\begin{theorem}
\label{theorem:1}
(\textbf{Value Consistency in Virtual Expansion}): Given $r \in (0, 1)$, confidence $\delta \in (0, 1)$, finite action set $\mathcal{A}$.
$\exists N_0, \forall N > N_0, k \ge rN$, let $\epsilon_k = \sqrt{\frac{1}{2k} \ln{\frac{100k^2}{\delta}}}$, we have
(1) After $k$ times vanilla expansion, $Pr\{\bigcap_{a \in \mathcal{A}} \left| \Bar{Q}_a^k - Q_a \right| < \epsilon_k\} 
        > (1 - \frac{e \delta |\mathcal{A}|}{50 k^2})$;
(2) After $k$ times vanilla expansion and $N - k$ times virtual expansion, $Pr\{\bigcap_{a \in \mathcal{A}} \left| \hat{Q}_a^N - Q_a \right| < \epsilon_k\}
        > (1 - \frac{e \delta |\mathcal{A}|}{50 r^2 N^2})$, where $e$ is the Euler's number.
\end{theorem}

\begin{proof}
Firstly, we have Hoeffding's inequality:
\begin{equation}
    \begin{aligned}
        & \forall i = 1, 2, \cdots, n, a_i \le X_i \le b_i, S_n = X_1 + X_2 + \cdots + X_N \\
        & Pr\{\left| S_n - \mathbb{E}[S_n] \right| \ge t\} \le 2 \exp{\frac{-2t^2}{\sum_{i=1}^n (b_i - a_i)^2}}
    \end{aligned}
\end{equation}

Observe that at step $k$, given confidence $\delta$, let $\epsilon_k = \sqrt{\frac{1}{2k} \ln{\frac{100k^2}{\delta}}}$, assumed that $R^t_a$ are independent and bounded in $[0, 1]$. Then for action $a$, we have 
\begin{equation}
\label{eq:hoff_general}
    \begin{aligned}
        Pr\{\left| \Bar{Q}_a^k - \mathbb{E}[\Bar{Q}_a^k] \right| \ge \epsilon_k\} &=
        Pr\{\left| \Bar{Q}_a^k - Q_a \right| \ge \epsilon_k\} = 
        Pr\{\left| \sum_{t=1}^{T_a^k} R_a^t - T_a^k Q_a \right| \ge \epsilon_k T_a^k \} \\
        &\le 2 \exp{(\frac{-2 (\epsilon_k T_a^k)^2}{T_a^k})}
        = 2 \exp{(-2 T_a^k \epsilon_k^2)} = \delta_{k, a}
    \end{aligned}
\end{equation}

From Lemma \ref{lemma:1}, we know that $1 \le T_a^k \le k$, we have $2 \exp{(-2 k \epsilon_k^2)} \le \delta_{k, a} \le 2 \exp{(-2 \epsilon_k^2)}$. After simplification, $\forall a \in \mathcal{A}$, we have 
\begin{equation}
\label{eq:delta_a}
    \frac{\delta}{50 k^2} \le \delta_{k, a} \le \frac{\delta}{50 k^2} \exp{(\frac{1}{k})}
\end{equation}

And we know that $\sum_{a \in \mathcal{A}} T_a^k = k$, so we have
\begin{equation}
    \begin{aligned}
        Pr\{\bigcap_{a \in \mathcal{A}} \left| \Bar{Q}_a^k - Q_a \right| \ge \epsilon_k\} 
        & \le \prod_{a \in \mathcal{A}} \delta_{k, a} \\
        &= \prod_{a \in \mathcal{A}} 2 \exp{(-2 T_a^k \epsilon_k^2)} \\
        &= 2 \exp{(-2 k \epsilon_k^2)} \\
        &= \frac{\delta}{50 k^2} = \delta_k
    \end{aligned}
\end{equation}

And from Eq. (\ref{eq:hoff_general}), we have $Pr\{\left| \Bar{Q}_a^k - Q_a \right| < \epsilon_k\} \ge 1 - \delta_{k, a}$, then
\begin{equation}
\label{eq:prob_ood}
    \begin{aligned}
        Pr\{\bigcap_{a \in \mathcal{A}} \left| \Bar{Q}_a^k - Q_a \right| < \epsilon_k\}
        & \ge \prod_{a \in \mathcal{A}} (1 - \delta_{k, a}) \\
        & \ge \prod_{a \in \mathcal{A}} (1 - \frac{\delta}{50 k^2} \exp{(\frac{1}{k})}) \\
        & = (1 - \frac{\delta}{50 k^2} \exp{(\frac{1}{k})})^{|\mathcal{A}|}
    \end{aligned}
\end{equation}
Consider the function $f(x) = (1 - x)^n - (1 - nx), x = \frac{\delta}{50 k^2} \exp{(\frac{1}{k})} \in (0, \frac{e \delta}{50}], n = |\mathcal{A}| >= 2$, we have $f^{'}(x) = -n(1 - x)^{n-1} + n$. Since $\delta < 1, k \ge 1$, $x \le \frac{e \delta}{50} < \frac{e}{50} < 1$. Then we have $f^{'}(x) > 0$. Therefore, we have $f(x) > f(0) = 0$ and $(1 - x)^n > (1 - nx)$. So
\begin{equation}
    \begin{aligned}
        Pr\{\bigcap_{a \in \mathcal{A}} \left| \Bar{Q}_a^k - Q_a \right| < \epsilon_k\}
        & = (1 - \frac{\delta}{50 k^2} \exp{(\frac{1}{k})})^{|\mathcal{A}|} \\
        & > (1 - \frac{\delta |\mathcal{A}|}{50 k^2} \exp{(\frac{1}{k})}) \\
        & > (1 - \frac{e \delta |\mathcal{A}|}{50 k^2})
    \end{aligned}
\end{equation}
So we have 
\begin{equation}
    \begin{aligned}
        & \lim_{k \rightarrow \infty} \epsilon_k = \lim_{k \rightarrow \infty} \sqrt{\frac{1}{2k} \ln{\frac{100k^2}{\delta}}} = 0, \\
        & \lim_{k \rightarrow \infty} \delta_k = \lim_{k \rightarrow \infty} \frac{\delta}{50 k^2} = 0, \\
        & \lim_{k \rightarrow \infty} (1 - \frac{e \delta |\mathcal{A}|}{50 k^2}) = 0, 
    \end{aligned}
\end{equation}
Therefore, we know that
\begin{equation}
\label{eq:res_1}
    \begin{aligned}
        Pr\{\bigcap_{a \in \mathcal{A}} \left| \Bar{Q}_a^k - Q_a \right| < \epsilon_k\} 
        &> (1 - \frac{e \delta |\mathcal{A}|}{50 k^2}), \epsilon_k = \sqrt{\frac{1}{2k} \ln{\frac{100k^2}{\delta}}} \\
        \lim_{k \rightarrow \infty} Pr\{\bigcap_{a \in \mathcal{A}} \left| \Bar{Q}_a^k - Q_a \right| = 0\} 
        &= 1
    \end{aligned}
\end{equation}
The probability can be converged to 1, and the convergence rate is $O(\frac{1}{k^2})$.

According to description of virtual expansion in Algo. \ref{alg:virtual_expand}, we know that after extra $N - k$ virtual expansion, the estimated $Q$-values keep the same as the $k$-step. This is because the visitation distributions of the previous $k$ steps are identical. 

Therefore, for virtual expansion, the Eq. (\ref{eq:res_1}) is also satisfied. 

Since $k \ge rN$,  tor the next $N - k$ steps, the empirical mean $Q$-values $\hat{Q}_a^N$ are equal to $\Bar{Q}_a^k$. 
So we have
\begin{equation}
    \begin{aligned}
        Pr\{\bigcap_{a \in \mathcal{A}} \left| \hat{Q}_a^N - Q_a \right| < \epsilon_k\}
        & > (1 - \frac{e \delta |\mathcal{A}|}{50 k^2})
        & > (1 - \frac{e \delta |\mathcal{A}|}{50 r^2 N^2})
    \end{aligned}
\end{equation}



\end{proof}

\begin{theorem}
(\textbf{Best Action Identification in Virtual Expansion}): Given $r \in (0, 1)$, confidence $\delta \in (0, 1)$, finite action set $\mathcal{A}$.
Suppose $\hat{Q}_a^N$ is the final empirical mean value of action after V-MCTS, $\Bar{Q}_a^N$ is the final empirical mean value of after vanilla MCTS. $a = 1$ is the action of the highest expected value, $a = *$ is the action of the highest empirical mean value.
$\exists N_0, \forall N > N_0, k \ge rN$, let $\epsilon_k = \sqrt{\frac{1}{2k} \ln{\frac{100k^2}{\delta}}}$, after V-MCTS, we have $Pr\{ \left| \hat{Q}_*^N - \Bar{Q}_1^N \right| < \epsilon_k + \epsilon_N\}
        > 1 - 2 (\frac{\delta}{50 k^2} \exp{(\frac{1}{1.61 \sqrt{k}})} + \frac{\delta}{50 N^2} \exp{(\frac{1}{N})})$.
\end{theorem}
\begin{proof}

From Eq. (\ref{eq:hoff_general}) in Theorem \ref{theorem:1}, we know that 
\begin{equation}
    \begin{aligned}
        &Pr\{ Q_1 - \Bar{Q}_1^k < \epsilon_k\} \ge 1 - \delta_{k, 1}, \\
        &Pr\{ \Bar{Q}_1^k - Q_1 < \epsilon_k\} \ge 1 - \delta_{k, 1}
    \end{aligned}
\end{equation}
, where $\epsilon_k = \sqrt{\frac{1}{2k} \ln{\frac{100k^2}{\delta}}}, \delta_{k, a} = 2 \exp{(-2 T_a^k \epsilon_k^2)}$
Besides, we know that $\forall a \in \mathcal{A}, Q_a \le Q_1$, so we have
\begin{equation}
    \begin{aligned}
        &Pr\{ Q_1 - \Bar{Q}_*^k < \epsilon_k\} \ge 1 - \delta_{k, *}, \\
        &Pr\{ \Bar{Q}_*^k - Q_1 < \epsilon_k\} \ge 1 - \delta_{k, *}
    \end{aligned}
\end{equation}
For different step $k, N$, we have
\begin{equation}
    \begin{aligned}
        &Pr\{ \Bar{Q}_*^k - \Bar{Q}_1^N < \epsilon_k + \epsilon_N\} \ge 1 - \delta_{k, *} - \delta_{N, 1}, \\
        &Pr\{ \Bar{Q}_1^k - \Bar{Q}_*^k < 2\epsilon_k\} \ge 1 - \delta_{k, *} - \delta_{N, 1} \\
        \Rightarrow 
        &Pr\{ \left| \Bar{Q}_*^k - \Bar{Q}_1^N \right| < \epsilon_k + \epsilon_N\} \ge 1 - 2 (\delta_{k, *} + \delta_{N, 1}) 
    \end{aligned}
\end{equation}

Before finding the bound of $\delta_{k, *}$, let us make an assumption first.

\begin{assumption}
Suppose that $\forall a \in \mathcal{A},$ $M_a = P(s, a) \in (0, 1)$ is the prior score obtained from the learned neural networks, we have $M_* \ge \frac{1}{|\mathcal{A}|}\sum_{a \in \mathcal{A}} M_a = \frac{1}{|\mathcal{A}|}$, where $* := \arg\max\limits_a \Bar{Q}_a^k$.
\end{assumption}

Here, this inequality is true when the learned neural networks can estimate the prior of the actions after training for some trials. In such a case, for the best empirical action $*$, the predicted prior score should be larger than the mean prior scores. 

From the Lemma 2 in P-UCT \citep{rosin2011multi}, we know that at most $\frac{1.61 \sqrt{n}}{M_*}$ distinct arms are pulled during the episode, where $M_*$ is the prior score $P(s, a)$ of the best action. $M_*$ is a constant during the search loop and we know that $M_* \ge \frac{1}{|\mathcal{A}|}\sum_{a \in \mathcal{A}} M_a = \frac{1}{|\mathcal{A}|}$.
An action will be selected for more times with a higher empirical mean values.
Therefore, for the empirical best action $*$, it has been selected more than $k / \frac{1.61 \sqrt{k}}{M_*} = \frac{M_*}{1.61}\sqrt{k}$ times, which means $\frac{M_*}{1.61}\sqrt{k} \le T_*^k \le k$. So we have 
\begin{equation}
\label{eq:delta_*}
    \frac{\delta}{50 k^2} \le \delta_{k, *} \le \frac{\delta}{50 k^2} \exp{(\frac{M_*}{1.61 \sqrt{k}})}
\end{equation}

From Eq. (\ref{eq:delta_*}) and (\ref{eq:delta_a}), we know that
\begin{equation}
    \begin{aligned}
        Pr\{ \left| \Bar{Q}_*^k - \Bar{Q}_1^N \right| < \epsilon_k + \epsilon_N\}
        & \ge 1 - 2 (\delta_{k, *} + \delta_{N, 1}) \\
        & \ge 1 - 2 (\frac{\delta}{50 k^2} \exp{(\frac{M_*}{1.61 \sqrt{k}})} + \frac{\delta}{50 N^2} \exp{(\frac{1}{N})})
    \end{aligned}
\end{equation}

According to description of virtual expansion in Algo. \ref{alg:virtual_expand}, we know that after extra $N - k$ virtual expansion, the estimated $Q$-values keep the same as the $k$-step. Compared with the visitations of vanilla MCTS and V-MCTS, the only difference is the empirical mean $Q$-values. Observed that $\hat{Q}_a^N = \hat{Q}_a^k$ when $k$ satisfies the VET-rule.

Consequently, for the V-MCTS, we have
\begin{equation}
    \begin{aligned}
        Pr\{ \left| \hat{Q}_*^N - \Bar{Q}_1^N \right| < \epsilon_k + \epsilon_N\} 
        \ge 1 - 2 (\frac{\delta}{50 k^2} \exp{(\frac{M_*}{1.61 \sqrt{k}})} + \frac{\delta}{50 N^2} \exp{(\frac{1}{N})}) \\
        > 1 - 2 (\frac{\delta}{50 k^2} \exp{(\frac{1}{1.61 \sqrt{k}})} + \frac{\delta}{50 N^2} \exp{(\frac{1}{N})})
    \end{aligned}
\end{equation}
\end{proof}

\begin{theorem}
(\textbf{Error Bound of V-MCTS}): Given $r \in (0, 1)$, confidence $\delta \in (0, 1)$, finite action set $\mathcal{A}$.
Suppose the virtual expanded policy $\hat{\pi}_k$ is generated from Algorithm \ref{alg:stop} (V-MCTS), $\exists N_0, \forall N > N_0, k \ge rN$, $\forall \epsilon \in [0, 1]$, we have:
if $\hat{\Delta}_s(k, k/2) < \epsilon$, $Pr\left\{ \left|\left| \pi_N(s) - \hat{\pi}_{k}(s) \right|\right|_{1} < 3 \epsilon \right\} > 1 - \frac{e \delta |\mathcal{A}|}{50 N^2} (1 + \frac{4}{r^2})$, where $e$ is the Euler's number.
\end{theorem}
\begin{proof}
Suppose that $k$ satisfies the VET-rule, which means $k \ge rN, \epsilon \in [0, 1], \hat{\Delta}_s(k, k/2) = \left|\left| \hat{\pi}_k(s) - \hat{\pi}_{k/2}(s) \right|\right|_{1} < \epsilon$. Here, it is obvious that the given $\epsilon$ is in a range of $[0, 1]$ because $\hat{\pi}_k(s)$ is a probability distribution.

In general, given the expected values $Q_a$ of each action $a$, assume there exists a ground truth policy $\pi(s)$, which does MCTS for N times given the expected values $Q_a$.

Then we have 
\begin{equation}
\label{eq:policy_dist_1}
\hat{\Delta}_s(N, k) = \left|\left| \hat{\pi}_N(s) - \hat{\pi}_{k}(s) \right|\right|_{1} \le \left|\left| \hat{\pi}_N(s) - \pi(s) \right|\right|_{1} + \left|\left| \hat{\pi}_k(s) - \pi(s) \right|\right|_{1}.
\end{equation}

\begin{assumption}
Suppose that given $\epsilon, r \in (0, 1)$, $\exists \sigma_{\epsilon}, N_0 > 0$, $\forall a \in \mathcal{A}$,
$\forall N > N_0, k \ge rN$,
when $\bigcap_{a \in \mathcal{A}} \left| \Bar{Q}_a^k - Q_a \right| < \sigma_{\epsilon}$, we have $\left|\left| \hat{\pi}_k(s) - \pi(s) \right|\right|_{1} < \epsilon$.
\end{assumption}

This assumption shows that when the L1 difference between all empirical mean values and the corresponding expected values, the difference of policy between $\pi(s)$ and $\hat{\pi}_k(s)$ can be bounded with the given distance $\epsilon$. This is obvious because virtual MCTS will do virtual expansion for the next $N - k$ times without changing the empirical mean values. Therefore, when $\sigma_{\epsilon}$ is small enough, during the next $N - k$ times expansion, the ucb scores of virtual expansion are similar to those of vanilla expansion with expected values. For example, when $\sigma_{\epsilon} \rightarrow 0$, $\Bar{Q}_a^k \rightarrow Q_a$, the virtual expansion is totally the same as the vanilla expansion with expected values. Then $\exists N_0, \forall N > N_0$, $\left|\left| \hat{\pi}_k(s) - \pi(s) \right|\right|_{1} = \left|\left| \hat{\pi}_N(s) - \pi(s) \right|\right|_{1} = 0 < \epsilon$. 

From Eq. (\ref{eq:res_1}) in Theorem \ref{theorem:1}, we have $Pr\{\bigcap_{a \in \mathcal{A}} \left| \Bar{Q}_a^N - Q_a \right| < \epsilon_N\} > (1 - \frac{e \delta |\mathcal{A}|}{50 N^2})$, where $\epsilon_N = \sqrt{\frac{1}{2N} \ln{\frac{100N^2}{\delta}}}$. Since $\exists N_1, \forall N > N_1$, $\sigma_{\epsilon}$ is a constant when $\epsilon$ is given, so $\sigma_\epsilon > \epsilon_N$, then with at least probability of $(1 - \frac{e \delta |\mathcal{A}|}{50 N^2})$
\begin{equation}
\label{eq:policy_dist_2}
\left|\left| \hat{\pi}_N(s) - \pi(s) \right|\right|_{1} < \epsilon. 
\end{equation}

Since we know that $\left|\left| \hat{\pi}_k(s) - \hat{\pi}_{k/2}(s) \right|\right|_{1} < \epsilon$,  $\left|\left| \hat{\pi}_{k}(s) - \pi(s) \right|\right|_{1} - \left|\left| \hat{\pi}_{k/2}(s) - \pi(s) \right|\right|_{1} \le \left|\left| \hat{\pi}_k(s) - \hat{\pi}_{k/2}(s) \right|\right|_{1} < \epsilon$.

From Eq. (\ref{eq:res_1}) in Theorem \ref{theorem:1}, we have
\begin{equation}
    \begin{aligned}
        Pr\{\bigcap_{a \in \mathcal{A}} \left| \Bar{Q}_a^{k/2} - Q_a \right| < \epsilon_{k/2}\} & > (1 - \frac{4 e \delta |\mathcal{A}|}{50 k^2})
    \end{aligned}
\end{equation}

We know that $k \ge rN$, $\exists N_2, \forall N > N_2$, $\epsilon_k \le \sqrt{\frac{1}{2rN} \ln{\frac{100 r^2 N^2}{\delta}}} < \sigma_{\epsilon}$, then with at least probability of $(1 - \frac{4 e \delta |\mathcal{A}|}{50 k^2})$, $\left|\left| \hat{\pi}_{k/2}(s) - \pi(s) \right|\right|_{1} < \epsilon$ and
\begin{equation}
    \begin{aligned}
        \left|\left| \hat{\pi}_{k}(s) - \pi(s) \right|\right|_{1} \le \left|\left| \hat{\pi}_{k/2}(s) - \pi(s) \right|\right|_{1} + \left|\left| \hat{\pi}_k(s) - \hat{\pi}_{k/2}(s) \right|\right|_{1} < 2 \epsilon
    \end{aligned}
\end{equation}

Back to Eq. (\ref{eq:policy_dist_1}), with at least $(1 - \frac{e \delta |\mathcal{A}|}{50 N^2}) \times (1 - \frac{4 e \delta |\mathcal{A}|}{50 k^2})$, we have
\begin{equation}
    \begin{aligned}
        \hat{\Delta}_s(N, k) & = \left|\left| \hat{\pi}_N(s) - \hat{\pi}_{k}(s) \right|\right|_{1} \\
        & \le \left|\left| \hat{\pi}_N(s) - \pi(s) \right|\right|_{1} + \left|\left| \hat{\pi}_k(s) - \pi(s) \right|\right|_{1} \\
        & < \epsilon + 2\epsilon = 3\epsilon
    \end{aligned}
\end{equation}

For the $N$-th iteration of the search process, the final visitation distributions keep the same between the original expansion (Algorithm \ref{alg:expand}) and the virtual expansion (Algorithm \ref{alg:virtual_expand}). This is because at the last iteration, searching the nodes after the root has no effects on the final distribution. Therefore, $\hat{\pi}_N(s) = \pi_N(s)$. So we have
\begin{equation}
    \begin{aligned}
        \left|\left| \pi_N(s) - \hat{\pi}_{k}(s) \right|\right|_{1} 
        & = \left|\left| \hat{\pi}_N(s) - \hat{\pi}_{k}(s) \right|\right|_{1} \\
        & \le \left|\left| \hat{\pi}_N(s) - \pi(s) \right|\right|_{1} + \left|\left| \hat{\pi}_k(s) - \pi(s) \right|\right|_{1} \\
        & < \epsilon + 2\epsilon = 3\epsilon
    \end{aligned}
\end{equation}

Therefore, let $N_0 = \max \{N_1, N_2\}$, $\forall N > N_0$,
\begin{equation}
    \begin{aligned}
        Pr\{ \left|\left| \pi_N(s) - \hat{\pi}_{k}(s) \right|\right|_{1}  < 3 \epsilon \} 
        & >= (1 - \frac{e \delta |\mathcal{A}|}{50 N^2}) \times (1 - \frac{4 e \delta |\mathcal{A}|}{50 k^2}) \\
        & > 1 - (\frac{e \delta |\mathcal{A}|}{50 N^2} + \frac{4 e \delta |\mathcal{A}|}{50 k^2}) \\
        & = 1 - \frac{e \delta |\mathcal{A}|}{50} (\frac{1}{N^2} + \frac{4}{k^2}) \\
        & \ge 1 - \frac{e \delta |\mathcal{A}|}{50 N^2} (1 + \frac{4}{r^2}) \\
    \end{aligned}
\end{equation}


\end{proof}

\end{document}